\def\year{2023}\relax
\documentclass[letterpaper]{article} % DO NOT CHANGE THIS
\usepackage{aaai23}  % DO NOT CHANGE THIS
\usepackage{times}  % DO NOT CHANGE THIS
\usepackage{helvet}  % DO NOT CHANGE THIS
\usepackage{courier}  % DO NOT CHANGE THIS
\usepackage[hyphens]{url}  % DO NOT CHANGE THIS
\usepackage{graphicx} % DO NOT CHANGE THIS
\usepackage{natbib}  % DO NOT CHANGE THIS AND DO NOT ADD ANY OPTIONS TO IT
\usepackage{xcolor}
\usepackage{caption} % DO NOT CHANGE THIS AND DO NOT ADD ANY OPTIONS TO IT
\DeclareCaptionStyle{ruled}{labelfont=normalfont,labelsep=colon,strut=off} % DO NOT CHANGE THIS
\usepackage{subcaption}

\usepackage{algorithm}
\usepackage{algorithmic}
\usepackage{amsmath,amsthm,amsfonts,amssymb}

\input{epsf}

\newtheorem{theorem}{Theorem}
\newtheorem{lemma}{Lemma}

\newtheorem{remark}{Remark}
\newtheorem{definition}{Definition}
\newtheorem{assumption}{Assumption}
\usepackage{newfloat}
\usepackage{listings}
\floatstyle{ruled}
\newfloat{listing}{tb}{lst}{}
\floatname{listing}{Listing}
\title{Recovering the Graph Underlying Networked Dynamical Systems \\ under Partial Observability: A Deep Learning Approach}
\author{
    Sérgio Machado\textsuperscript{\rm 2,}\textsuperscript{\rm 5}, Anirudh Sridhar\textsuperscript{\rm 3}, Paulo Gil\textsuperscript{\rm 2,}\textsuperscript{\rm 4} , Jorge Henriques\textsuperscript{\rm 2}\\José M. F. Moura\textsuperscript{\rm 5}, Augusto Santos\textsuperscript{\rm 1,}\textsuperscript{\rm 2}\footnote{Corresponding author. E-mail: augusto.pt@gmail.com.}
}
\affiliations{
    \textsuperscript{\rm 1} Instituto de Telecomunica\c{c}\~{o}es, Portugal\\

    \textsuperscript{\rm 2} University of Coimbra, Portugal\\
    \textsuperscript{\rm 3} Department of Electrical and Computer Engineering at Princeton University, New Jersey, NJ, USA\\

    \textsuperscript{\rm 4} Universidade Nova de Lisboa, Portugal\\

    \textsuperscript{\rm 5} Department of Electrical and Computer Engineering at Carnegie Mellon University, Pittsburgh, PA, USA\\

    smachado@student.dei.uc.pt, asridhar@princeton.edu, $\left\{\right.\!\!\!$ psg,jh$\left.\!\right\}$@dei.uc.pt, moura@andrew.cmu.edu,\\ augusto.pt@gmail.com
}

\begin{document}

\maketitle

\begin{abstract}
We study the problem of graph structure identification, i.e., of recovering the graph of dependencies among time series. We model these time series data as components of the state of linear stochastic \emph{networked} dynamical systems. We assume partial observability, where the state evolution of only a subset of nodes comprising the network is observed. We propose a new feature-based paradigm: to each pair of nodes, we compute a feature vector from the observed time series. We prove that these features are linearly separable, i.e., there exists a hyperplane that separates the cluster of features associated with connected pairs of nodes from those of disconnected pairs. This renders the features amenable to train a variety of classifiers to perform causal inference. In particular, we use these features to train Convolutional Neural Networks (CNNs). The resulting causal inference mechanism outperforms state-of-the-art counterparts w.r.t. sample-complexity. The trained CNNs generalize well over structurally distinct networks (dense or sparse) and noise-level profiles. Remarkably, they also generalize well to real-world networks while trained over a synthetic network -- namely, a particular realization of a random graph. %Finally, the proposed method consistently reconstructs %the graph in a pairwise manner, that is, by deciding if an edge or arrow is present or absent in each pair of nodes, from the corresponding time series of each pair. This fits the framework of large-scale systems, where observation or processing of all nodes in the network is prohibitive.

%Finally, the proposed method consistently reconstructs %the graph in a pairwise manner, that is, by deciding if an edge or arrow is present or absent in each pair of nodes, from the corresponding time series of each pair. This fits the framework of large-scale systems, where observation or processing of all nodes in the network is prohibitive.
\end{abstract}

\section{Introduction}

\emph{Networked} dynamical systems are characterized by a set of interconnected nodes or agents. The state of the nodes evolves over time according to their peer-to-peer interactions constrained by a support network of contacts~\cite{Barrat,liggett,Queues,porterdynamical}. More concretely, the state of a node $i$ is only \emph{immediately} affected by the state of nodes that directly link to $i$, i.e., nodes that bear a direct causal effect on node $i$. This causal network is captured by a graph, which is often a latent structure underlying these systems.

Examples of networked dynamical systems include: $i$)~\textit{Pandemics} -- the fraction of infections within each community of individuals is captured by a time series that is strongly influenced by contacts in neighboring communities. Knowledge of the contact network (which determines the main avenues of contagion) is critical for designing effective mitigation measures \cite{singlevirus,topoepidemics,augusto_qualitative,Net_dismantling,Net_dismantling2}. For example, a natural mitigation policy is \emph{network dismantling}: aiming to quarantine a minimal set of nodes to promote a maximal disconnect of the underlying contagion network~\cite{Net_dismantling,Net_dismantling2} -- thus, hindering virus propagation across communities without disrupting the global function of the networked system; $ii$)~\textit{Brain activity} -- based on temporal signals gathered from cranial probes, an important task is to infer the so-called \emph{Functional Connectivity Matrix}, which represents the graph of interactions among the active regions of the brain (see, e.g.,~\cite{brainaugusto}). Recent evidence shows that the Functional Connectivity Matrix can be used to diagnose or predict the onset of motor activities or cognitive disorders~\cite{epi_pred,Brain3,Brain2,alzheimer,parkinson,net_epi,Net_body}; $iii$)~\textit{Finance} -- the dynamics of stock prices can be influenced by interactions between firms, and knowledge of this interaction network can inform government interventions, for instance \cite{finance1,finance2,finance3}.

In most practical instances of the examples above, the node-level time series are readily accessible, but the underlying causal network -- which is of fundamental importance in downstream tasks -- is fully or partially unknown. To address this issue, a growing body of literature has developed methods for reconstructing the network from the observed node-level time series~\cite{R1minusR3,NIPS_Bento,MaterassiSalapakaCDC2015,tomo_journal_proceedings}. In this work, we focus on \emph{linear stochastic networked dynamical systems}, which is arguably one of the most natural settings for network identification from time series since a great class of nonlinear networked dynamical systems can be addressed via linearization about the equilibria under small-noise regimes~\cite{Ching2017ReconstructingLI,Napoletani} or via appropriate embedding in higher dimensional spaces~\cite{nonlinearLim2015,Koopman_Gon}. Moreover, since it is typically impractical to monitor \emph{all} node-level signals in large-scale systems, we assume a \emph{partial observability} setting, wherein we observe the time series corresponding to a small subset of nodes and aim to reconstruct the corresponding subgraph connecting them using a small number of samples. This task is much more challenging than the \emph{full observability} case, since the time series of the observed nodes are also affected by the unobserved dynamics of the remainder of the network. Fig.~\ref{fig:tomography} summarizes the structure identification framework considered.

This work departs from the standard approach of reconstructing networks based on scalar measures between time series -- i.e., measures that assign a real-value to the coupling-strength between nodes, e.g., correlation, Granger, Precision matrix, etc. -- which dates back to~\cite{ChowLiu}. We propose a novel \emph{feature vector based setting} constructed for each pair of nodes from their time series. In this novel setting, structure identification leverages the separability properties of the proposed feature vectors in a higher-dimensional space. We provide rigorous theoretical results proving that our features are \emph{linearly separable} for \emph{any} undirected network, once sufficiently many samples are taken. This provides the explainability of the method: our features can be readily used as input to a variety of machine learning pipelines to perform causal inference. We demonstrate that CNNs trained with our features outperform state-of-the-art methods in terms of accuracy and sample complexity.

\begin{figure} [!t]
	\begin{center}
		\includegraphics[scale= 0.38]{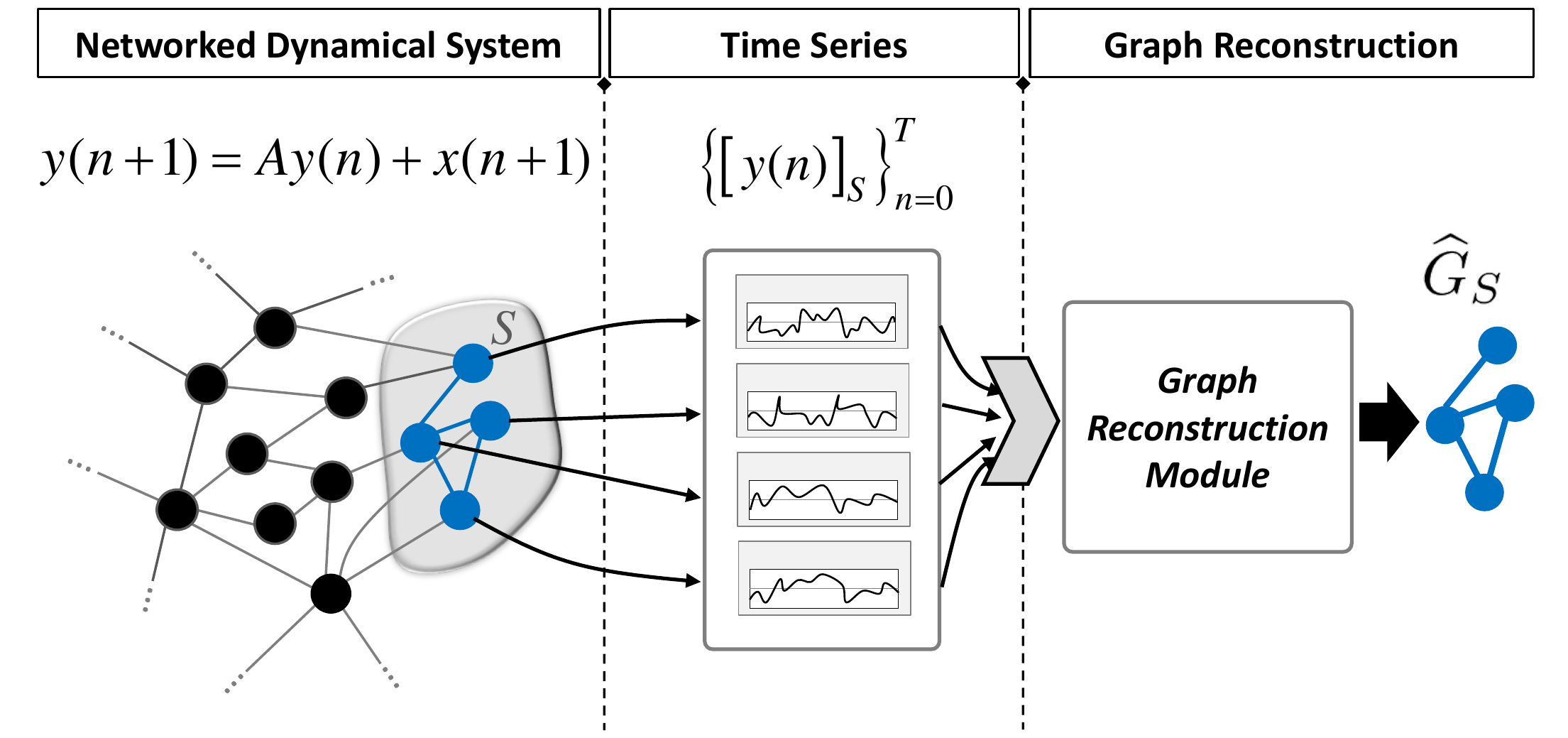}
		\caption{Structure identification under partial observability.}
		\label{fig:tomography}
	\end{center}
%\vspace{-10px}
\end{figure}

\section{Related Work}

Causal inference relies on the nature of the samples, and it depends on whether the observed time series are drawn independently from multivariate distributions (often assumed i.i.d. within the scope of graphical models), or are time series stemming from some networked dynamical law (not i.i.d.). For the multivariate case, the Markov property establishes a one-to-one correspondence between certain probability distributions and the set of (possibly directed) graphs~\cite{Hammersley,pearl_2009}. If $X$ is conditionally independent of $Z$ given $Y$, then there is no arrow, or direct causal effect, from $Z$ to $X$. The dependence relationships are thus captured by a graph. This Markov property can be extended to discrete-time networked dynamical systems: the state of a node at instant $n+1$ depends only on the state of some nodes at time $n$ (also known as neighbors or parents). The general goal of causal inference is to uncover the possible \emph{avenues of information flow}: to recover from observation of the time series samples the underlying graph structure of dependencies defined by the Markov property. Typically, this is done by leveraging various forms of scalar measures between time series, e.g., transformations of the covariance matrix; regression (e.g., Granger estimator)~\cite{Geigeretal15,NIPS_Bento,tomo_journal_proceedings}; or other scalar graph metrics~\cite{MaterassiSalapakaCDC2015}. The performance of the precise method ties strongly to the data generative process and to whether the system is fully or partially observed.

\subsection{Full-Observability}

\textbf{Graphical models.} Classical algorithms (assuming i.i.d. samples) based on conditional independence tests, include the SGS~\cite{Spirtes2000}, PC~\cite{spirtes91}, GES~\cite{David_pc}, and FGS~\cite{Ramsey2016AMV}. The algorithms and sufficient conditions for consistency devised in these works rely on sparsity related structural constraints that hardly fit the connectivity pattern of general networks. The work~\cite{AnandkumarWalkSummabilityJMLR} offers an approach for the large scale setting under certain complying assumptions of sparsity. The independence tests are leveraged via conditional covariance tests. All structural constraints, revolving around sparsity, play a critical role to render the scaling of the independence tests amenable to computation, otherwise, the problem becomes quickly unfeasible~\cite{Breslerhardness,complexity_markov}.

\textbf{Networked dynamical systems.} For approaches in the signal processing literature,~\cite{Mateos} provides an overview highlighting regression plus regularization of the network sparsity methods for full-observability over distinct models -- primarily promoting sparsity of the latent network, -- including linear dynamical systems as vector autoregressive (VAR) models, as in~\cite{mei,Moneta,TAC_graph_learning}. In this regard,~\cite{NIPS_Bento} addresses the problem for linear stochastic differential equations (SDEs) via an optimization formulation that regularizes sparsity of the latent network. This problem is addressed by first converting the continuous-time SDE into a discrete-time linear dynamical system -- a technique that yields the discrete-time model considered in this work. Other schemes exploit spectral-based methods~\cite{Granger_paper, Santiago,Santiago_template}. These leverage the spectral properties of the interaction matrix or support graph~\cite{Aliaksei} to characterize signatures that allow consistent estimation over certain sparse networks.

\subsection{Partial-Observability}

\textbf{Graphical models.} In general, the proposed approaches rely on conditional independence (CI) tests or measures thereof, e.g., conditional mutual information (CMI) or transfer entropy, and a causal link is declared whenever a test yields a positive CMI-based metric. Classical algorithms for causal inference under the presence of latent variables are the FCI~\cite{spirtes_latent_95} and RFCI~\cite{Colombo_latent}. As in the full-observability setting, consistent tests scale combinatorially with the connectivity of the causal graph, rendering the CI-based approaches impractical for denser graphs. To control the curse of connectivity, CI-based methods often act at a microscopical level relying on several strong structural constraints including, directed acyclic graphs, long girth~\cite{anandkumar13,AnandkumarValluvanAOS} and other more technical local structural conditions, such as
bottleneck and non-redundancy~\cite{Causal_part_Nips,pmlr_causal}.

\textbf{Networked dynamical systems.} In~\cite{MaterassiSalapakaCDC2015,MaterassiSalapakaCDC2012,MaterassiSalapakaTAC2012}, linear dynamical systems are addressed via certain pseudo-metrics, e.g. log-coherence distance, aiming to capture the true graph-distance between nodes. %It is shown that some pairs of nodes can be classified consistently, as long as the network meets some
%strict conditions of sparsity, e.g., no undirected cycles are allowed (polytree-like graphs).
In~\cite{Geigeretal15} some conditions on the network connectivity and interaction matrix of a linear networked dynamical system are proposed, in order to obtain uniqueness of the network connectivity given partially observed samples. It does not provide, however, an algorithm with consistency guarantees to retrieve the uniquely determined network. On the other hand, the work~\cite{POIM} uses an expectation-maximization based approach to address certain discrete-time discrete state-space networked dynamical systems, while~\cite{ChandrasekaranParriloWillsky, Jalali, Mei_latent} resort to convex optimization based methods for regularizing the sparsity of the network under partial observability. The works~\cite{tomo_journal,tomo_journal_proceedings,open_Journal} establish structural consistency of the Granger (or regression) and other matrix-valued estimators over partially observed discrete-time linear stochastic networked dynamical systems with symmetric interaction matrices, for distinct regimes of network connectivity (including densely connected networks). Similar to~\cite{AnandkumarValluvanAOS}, the structural consistency of these estimators is established in the \emph{thermodynamic limit}, i.e., as the number of nodes scales to infinite, which fits the framework of large-scale networks.
Recently,~\cite{R1minusR3} proved that the underlying interaction matrix, up to a multiplicative constant related to the noise level, can be expressed as a linear combination of covariance matrices, with high probability, under the following regime: $i$) the interaction matrix $A$ is symmetric; $ii$) the noise $\mathbf{x}$ is \emph{diagonal} and homogeneous, i.e., its covariance matrix is a multiple of the identity matrix. Theorem~$1$ in~\cite{R1minusR3} will be used in the present work to establish an important result regarding the proposed set of feature vectors, namely, consistent linear separability. This property will further yield a competitive performance for the trained CNNs in terms of sample-complexity.

\section{Problem Formulation}
\label{sec:probform}

We consider the linear networked dynamical law
\begin{equation}\label{eq:model}
\mathbf{y}(n+1)=A\mathbf{y}(n)+\mathbf{x}(n+1),
\end{equation}
where~$\mathbf{y}(n)=\left[y_1(n)\,\,y_2(n)\,\,\ldots\,\,y_N(n)\right]^{\top}\in\mathbb{R}^N$ represents the state-vector of the $N$-dimensional networked dynamical system at time $n$ that collects the states~$y_i(n)$ of each node $i$ at time $n$; $\mathbf{x}(n)\sim\mathcal{N}\left(0,\sigma^2 I_N\right)$ represents the excitation noise associated with the $N$ nodes of the system with covariance matrix~$\sigma^2 I_N$, and independent across time~$n$; $A\in \mathbb{R}_{+}^{N\times N}$ refers to the non-negative interaction matrix whose support represents the underlying graph linking the nodes. The dynamical system is assumed to be stable, i.e., $\rho(A)<1$, where $\rho(A)$ stands for the spectral radius of $A$.

This work deals with the problem of recovering the support of the submatrix $A_{S}$, i.e., the graph structure of connections among the observed nodes in the subset $S$ from observation of the subvector $\left[\mathbf{y}(n)\right]_{S}=\left[\mathbf{y}_{m_1}(n)\,\,\mathbf{y}_{m_2}(n)\,\,\ldots\,\,\mathbf{y}_{m_{\left|S\right|}}(n)\right]^{\top}\in\mathbb{R}^{\left|S\right|}$ over time $n$, where $\left|S\right|$ is the cardinality of the subset~$S$ (see Fig.\ref{fig:tomography}).

\textit{Notation}: $S=\left\{m_1,m_2,\ldots,m_{\left|S\right|}\right\}\subset \left\{1,2,\ldots,N\right\}$ is a nonempty subset of indexes with $m_1<m_2<\ldots< m_{\left|S\right|}$ and $\left|S\right|\leq N$; given a vector $\mathbf{y}\in\mathbb{R}^{N}$,~$\left[\mathbf{y}\right]_{S}=\left[\mathbf{y}_{m_1}(n)\,\,\mathbf{y}_{m_2}(n)\,\,\ldots\,\,\mathbf{y}_{m_{\left|S\right|}}(n)\right]^{\top}$ is the subvector obtained from $\mathbf{y}$ and indexed by $S$; accordingly, a similar notation is adopted for matrices, namely, given $A\in\mathbb{R}^{N\times N}$, the matrix $A_{S}\in \mathbb{R}^{\left|S\right|\times \left|S\right|}$ or $\left[A\right]_{S}\in \mathbb{R}^{\left|S\right|\times \left|S\right|}$ is defined as the submatrix whose $ij^{\text{th}}$ entry is $A_{m_i m_j}$; ${\sf Supp}\left(A\right)$ is the support of the matrix $A$, i.e., $\left[{\sf Supp}\left(A\right)\right]_{ij}=\mathbf{1}_{\left\{A_{ij}\neq 0\right\}}$; $\left|\left|\mathbf{y}\right|\right|_{\infty}$ refers to the $L_{\infty}$-norm that returns the maximal absolute value across the entries of the vector $\mathbf{y}\in\mathbb{R}^{N}$; the set of natural numbers, including zero, is denoted by $\mathbb{N}=\left\{0,1,2,\ldots\right\}$.

\section{Structural Consistency}

Consider the following $k^{\text{th}}$ lag covariance matrix
\begin{equation}\label{eq:truecov}
R_k(n)\overset{\Delta}=\mathbb{E}\left[\mathbf{y}(n+k)\mathbf{y}(n)^{\top}\right]
\end{equation}
associated with the process~$\left(\mathbf{y}(n)\right)_{n\in\mathbb{N}}$. In addition, define the empirical counterpart of~$R_k(n)$
\begin{equation}
\widehat{R}_k(n)\overset{\Delta}=\frac{1}{n}\sum_{\ell=0}^{n-1} \mathbf{y}(\ell+k)\mathbf{y}(\ell)^{\top}.
\end{equation}

We refer to a matrix-valued estimator as any map whose input is given by the (observed) time series and the output is given by a matrix, namely,
\begin{equation}
\begin{array}{cccc}
F^{(n)}\,:\, & \mathbb{R}^{\left|S\right|\times n} & \longrightarrow & \mathbb{R}^{\left|S\right|\times \left|S\right|}\\
       & \left\{\left[\mathbf{y}(\ell)\right]_{S}\right\}_{\ell=0}^{n-1} & \longmapsto & \mathcal{F}^{(n)}
\end{array},
\end{equation}
for any given $n\in\mathbb{N}$. The idea is that the $ij^{\text{th}}$ entry of the output matrix~$\mathcal{F}^{(n)}$ estimates the strength of the link from $i$ to $j$ from $n$ samples of the observed time series. For instance, the empirical covariance matrix~$\widehat{R}_k(n)$, under full-observability, or $\left[\widehat{R}_k(n)\right]_{S}$, in the case of partial-observability, are examples of matrix-valued estimators.

\begin{definition}[structural consistency of a matrix]
A matrix-valued estimator~$F^{(n)}$ is structurally consistent with high probability, whenever there exists a threshold~$\tau$ so that,
\begin{equation}
\mathbb{P}\left(\mathcal{F}^{(n)}_{ij} > \tau\right)\overset{n\rightarrow \infty}\longrightarrow 1 \Longleftrightarrow i\rightarrow j,
\end{equation}
\noindent i.e., $i$ links to $j$ if and only if the $ij^{\text{th}}$ entry of the estimator matrix $\mathcal{F}^{(n)}$ lies above the threshold $\tau$, provided that there is a large enough number of samples~$n$.
\end{definition}

In other words, up to a proper threshold~$\tau$, the output matrix~$\mathcal{F}^{(n)}$ reflects the underlying structure of the graph in that $\left[{\sf Supp}(A_{S})\right]_{ij}=\mathbf{1}_{\left\{\mathcal{F}^{(n)}_{ij}>\tau\right\}}$, for all pairs $i\neq j$ w.h.p.

An example of a structurally consistent w.h.p. matrix-valued estimator (under partial observability) is given by~$\mathcal{F}^{(n)}\overset{\Delta}=\widehat{R}_1(n)-\widehat{R}_3(n)$~\cite{R1minusR3}. Other examples of matrix-valued estimators that are provably structurally consistent under partial observability include: $i$) \textbf{Granger} $\left[\widehat{R}_1(n)\right]_{S}\left(\left[\widehat{R}_0(n)\right]_{S}\right)^{-1}$; $ii$) \textbf{One-lag} $\left[\widehat{R}_1(n)\right]_{S}$; $iii$) \textbf{Residual} $\left[\widehat{R}_1(n)\right]_{S}- \left[\widehat{R}_0(n)\right]_{S}$. These latter estimators are proven to be structurally consistent under a certain \emph{thermodynamic} limit regime~\cite{open_Journal}, i.e., structural consistency is met in the limit $N \longrightarrow \infty$ with $\left|S\right|/N\longrightarrow \xi>0$ or with $\left|S\right|/N\longrightarrow 0$ for certain sparse regimes~\cite{tomo_journal}.
\begin{remark}
Technically, one should formally refer to the sequence~$\left(F^{(n)}\right)_{n\in\mathbb{N}}$ of maps (estimators) as structurally consistent with high probability. However, hereby, for the sake of simplicity it will be simply referred to as ``the estimator $F^{(n)}$ is structurally consistent w.h.p.".
\end{remark}
Next, we introduce a tensor-valued estimator which is, formally, any map whose input is given by the (observed) time series and the output is an order-$3$ tensor, as follows
\begin{equation}
\begin{array}{cccc}
T^{(n)}\,:\, & \mathbb{R}^{\left|S\right|\times n} & \longrightarrow & \mathbb{R}^{\left|S\right|\times \left|S\right| \times M}\\
       & \left\{\left[\mathbf{y}(\ell)\right]_{S}\right\}_{n=0}^{n-1} & \longmapsto & \mathcal{T}^{(n)}
\end{array},
\end{equation}
where the $ij^{\text{th}}$ entry of the order-$3$ tensor $\mathcal{T}^{(n)}$ is a vector $\mathcal{T}^{(n)}_{ij}\in\mathbb{R}^{M}$ that models a feature statistical descriptor corresponding to the pair~$ij$ in the network and that is built from $n$ samples of the time series~$\left\{\left[\mathbf{y}(\ell)\right]_{S}\right\}_{\ell=0}^{n-1}$.

\begin{figure} [!t]
\begin{center}
\includegraphics[scale= 0.25]{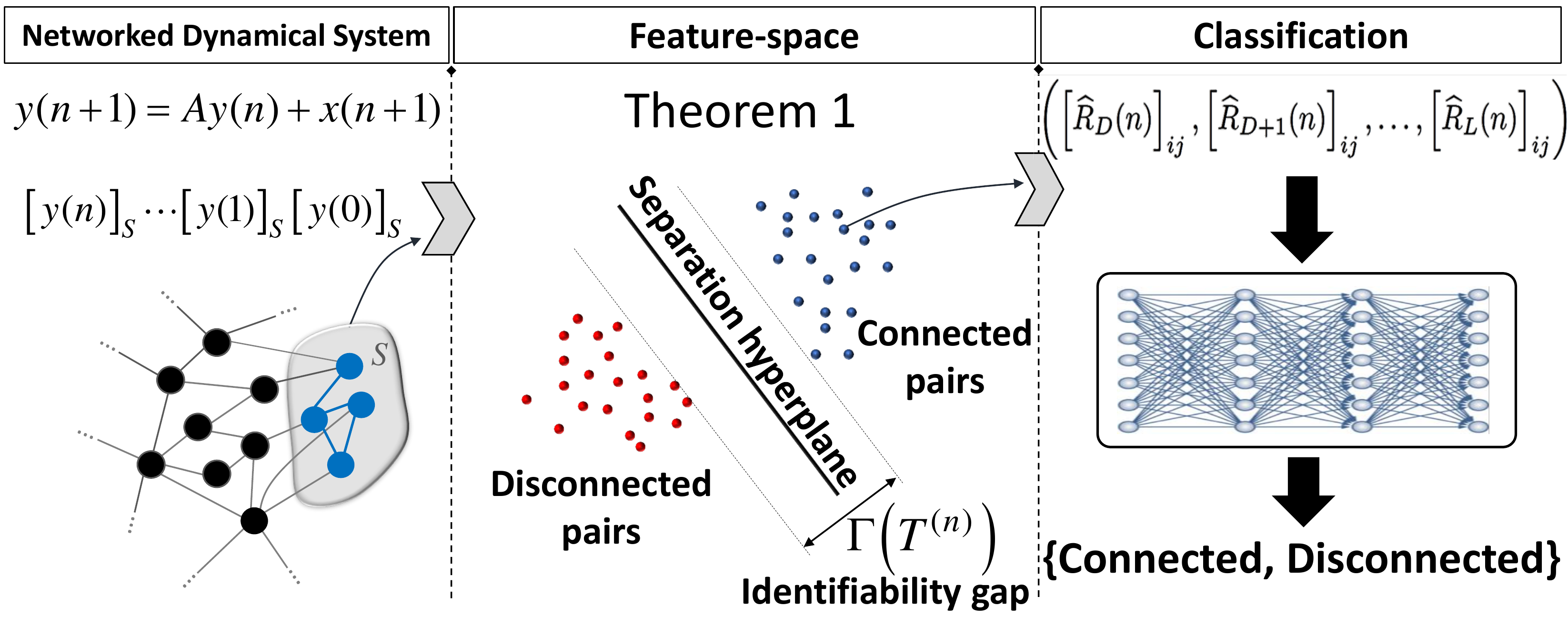}
\caption{Proposed framework.}
\label{fig:scheme}
\end{center}
%\vspace{-15px}
\end{figure}

\begin{definition}[structural consistency of a tensor]
A tensor-valued estimator~$T^{(n)}$ of order-$3$ is linearly structurally consistent with high probability, if there exists an affine map $\mathcal{L}\,:\,\mathbb{R}^{M}\rightarrow \mathbb{R}$ (or hyperplane) that separates the underlying features associated with connected pairs from those associated with disconnected pairs w.h.p., that is,
\begin{equation}
\begin{array}{cccl}
\mathbb{P}\left(\mathcal{L}(\mathcal{T}^{(n)}_{ij})  >  0\right) & \overset{n\rightarrow \infty}\longrightarrow & 1, & \mbox{ if } ij \mbox{ is connected,}\\
& & &\\
\mathbb{P}\left(\mathcal{L}(\mathcal{T}^{(n)}_{ij})  \leq  0\right) & \overset{n\rightarrow \infty}\longrightarrow & 1, & \mbox{ if } ij \mbox{ is disconnected}
\end{array}.
\end{equation}
\end{definition}

As an example, the estimator $T^{(n)}$ whose $ij^{\text{th}}$ entry of the tensor output $\mathcal{T}^{(n)}$ is defined as
\begin{equation} \mathcal{T}^{(n)}_{ij}\overset{\Delta}=\left(\left[\widehat{R}_D(n)\right]_{ij},\left[\widehat{R}_{D+1}(n)\right]_{ij},\ldots,\left[\widehat{R}_L(n)\right]_{ij}\right)\nonumber
\end{equation}
corresponds to an order-$3$ tensor-valued estimator. As we will show in the next section, if~$D\leq 1$ and~$L\geq 3$, then this estimator is linearly structurally consistent w.h.p.

\section{Features Separability \& Explainability}

This section provides the explainability results underlying the ML approach for graph learning considered in this work.

\begin{assumption} \label{assumption_1}
Let $\mathcal{E}^{(n)}:=\left\{E^{(n)}_1,E^{(n)}_{2},\ldots,E^{(n)}_{M}\right\}$ be a family of matrix-valued estimators such that for some $\mathbf{w}:=\left(w_1,w_2,\ldots,w_M\right)\in\mathbb{R}^{M}$ with $\mathbf{w}\neq 0$, the linear combination $E^{(n)}(\mathbf{w})=\sum_{\ell=1}^{M} w_{\ell} E^{(n)}_{\ell}$ is a structurally consistent w.h.p. matrix-valued estimator for the dynamics~\eqref{eq:model}.
\end{assumption}

\begin{figure*}[!t]
	\centering
	\includegraphics[width=0.99\linewidth]{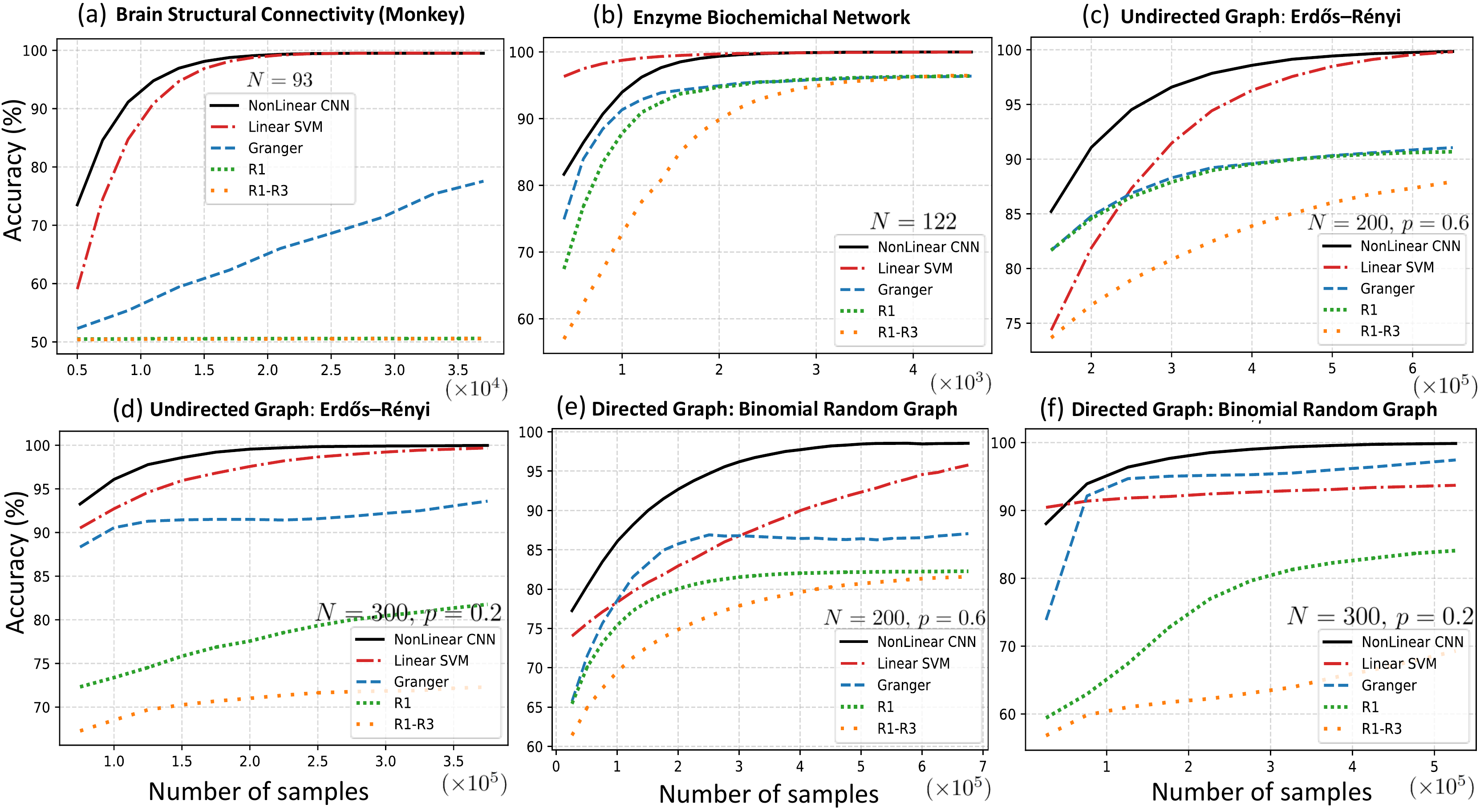}%0.97
	\caption{(a)-(f) Structure estimation performance: we plot the estimators' accuracy as a function of the number of samples. Plots (a)-(b) refer to real-world networks; (c)-(d) refer to undirected graphs (realization of an Erdős–Rényi model); and plots (e)-(f) refer to directed graphs (realization of a binomial random graph). We assume that we can only observe $\left|S\right|=20$ nodes.}
	\label{fig:Performance}
%	\vspace{-10px}
\end{figure*}

\begin{lemma}\label{lem:sep_linear}
For each pair $ij$, with $i\neq j$, define the associated feature vector as,
\begin{equation}
\mathcal{T}^{(n)}_{ij}:=\left(\left[E^{(n)}_1\right]_{ij},\left[E^{(n)}_2\right]_{ij},\ldots,\left[E^{(n)}_{M}\right]_{ij}\right)\in \mathbb{R}^{M}.
\end{equation}
Then, under Assumption~\ref{assumption_1}, the tensor-valued estimator $T^{(n)}$ is linearly structurally consistent w.h.p., or equivalently, the set of features $\left\{\mathcal{T}^{(n)}_{ij}\right\}_{i\neq j}\subset \mathbb{R}^{M}$ is consistently linearly separable w.h.p.
\end{lemma}

\begin{proof}
Since $E^{(n)}(\mathbf{w})=\sum_{\ell=1}^{M} w_{\ell} E^{(n)}_{\ell}$ is structurally consistent w.h.p. for some $\mathbf{w}\in\mathbb{R}^M$, then there exists a threshold $\tau_{\mathbf{w}}$ so that $\left[E^{(n)}(\mathbf{w})\right]_{ij}> \tau_{\mathbf{w}}$ across connected pairs $ij$ and $\left[E^{(n)}(\mathbf{w})\right]_{ij}< \tau_{\mathbf{w}}$, otherwise. Therefore, the affine map~$\mathcal{L}_{\mathbf{w}}(\mathbf{x})=\mathbf{x}\cdot \mathbf{w}-\tau_{\mathbf{w}}$ consistently separates the set of features~$\left\{\mathcal{T}^{(n)}_{ij}\right\}_{ij}$ w.h.p. Indeed,
\begin{equation}\label{eq:linearmap1}
\mathcal{L}_{\mathbf{w}}(\mathcal{T}^{(n)}_{ij}) =  \mathcal{T}^{(n)}_{ij}\cdot \mathbf{w}-\tau_{\mathbf{w}}=\left[E(\mathbf{w})\right]_{ij}-\tau_{\mathbf{w}}>0
\end{equation}
for a connected pair $ij$ or
\begin{equation}
\mathcal{L}_{\mathbf{w}}(\mathcal{T}^{(n)}_{ij})=\left[E^{(n)}(\mathbf{w})\right]_{ij}-\tau_{\mathbf{w}} < 0,
\end{equation}
otherwise. In other words, the hyperplane characterized by the linear map $\mathcal{L}_{\mathbf{w}}\,:\,\mathbb{R}^{M}\longrightarrow \mathbb{R}$ separates consistently the pairs $ij$ for all $i\neq j$, w.h.p.
\end{proof}

\begin{theorem} \label{Theorem_1}
For each pair $ij$, with $i\neq j$, define the associated feature vector as,
\begin{equation}
\mathcal{T}_{ij}^{(n)}:=\left(\left[\widehat{R}_D(n)\right]_{ij},\left[\widehat{R}_{D+1}(n)\right]_{ij},\ldots,\left[\widehat{R}_{L}(n)\right]_{ij}\right),\nonumber
\end{equation}
with $D\leq 1$ and $L \geq 3$, and assume that the interaction matrix~$A$ underlying the dynamics~\eqref{eq:model} is symmetric and the covariance matrix of the noise process $\left(\mathbf{x}(n)\right)_{n\in\mathbb{N}}$ is given by $\Sigma_x:=\sigma^2 I_N$, for some $\sigma>0$. Then, the set $\left\{\mathcal{T}_{ij}^{(n)}\right\}_{i\neq j}\subset \mathbb{R}^{M}$ is consistently linearly separable w.h.p.
\end{theorem}

\begin{proof}
Define the vector~$\mathbf{w}\in \left\{-1,0,1\right\}^{M}$ so that~$E^{(n)}(\mathbf{w})=\widehat{R}_1(n)-\widehat{R}_3(n)$, which is possible since $D\leq 1$ and $L\geq 3$. According to Theorem~\ref{Theorem_1} in~\cite{R1minusR3},~$E^{(n)}(\mathbf{w})=\widehat{R}_1(n)-\widehat{R}_3(n)$ is structurally consistent w.h.p. and the result now follows from the previous Lemma~\ref{lem:sep_linear}.
\end{proof}

\begin{remark}[Locality of the structural estimation] \label{remark_2}
Note that, to compute the feature~$\mathcal{T}^{(n)}_{ij}$ associated with each pair $ij$ defined in Theorem~$1$, we need only the time series~$\left\{y_i(\ell),y_j(\ell)\right\}_{\ell=0}^{n}$ associated with the pair $ij$ as
\begin{equation}
\mathcal{T}_{ij}^{(n)}:=\frac{1}{n}\sum_{\ell=0}^{n-1}\left(y_i(\ell+D)y_j(\ell),\ldots,y_i(\ell+M)y_j(\ell)\right),\nonumber
\end{equation}
\noindent which only involves information related to nodes $i$ and $j$. As such, it is possible to reconstruct the connectivity pattern in a pairwise manner. This is a special property that results from the fact that each lag-moment, or covariance matrix, in the feature vector can be locally estimated. Observe that the majority of the matrix-valued estimators does not exhibit this locality property. For example, to reconstruct the $ij^{\text{th}}$ entry of the Precision matrix $\left(\widehat{R}_0(n)\right)^{-1}$, one needs to know the whole matrix $\widehat{R}_0(n)$ (or a large portion around the pair $ij$ thereof). This has the drawback of implying the observation of a large set of nodes (or of the whole network) just to estimate the corresponding entry $ij$ of the Precision matrix.
\end{remark}

Now, given a matrix-valued estimator~$F^{(n)}$, define its \emph{identifiability gap} as~\cite{open_Journal}
\begin{equation}
\Gamma\left(F^{(n)}\right)\overset{\Delta}=\min_{ij\,:\,A_{ij}\neq 0} \mathcal{F}^{(n)}_{ij}- \max_{ij\,:\,A_{ij}=0} \mathcal{F}_{ij}^{(n)},
\end{equation}
i.e., the \emph{gap} between the smallest entry of $\mathcal{F}_{ij}^{(n)}$ across connected pairs and the largest entry of $\mathcal{F}^{(n)}_{ij}$ over disconnected pairs. An estimator~$F^{(n)}$ is structurally consistent w.h.p. if and only if $\Gamma\left(F^{(n)}\right)>0$ w.h.p., or in other words, if and only if connected pairs are separated from disconnected pairs, in view of the entries of the matrix $\mathcal{F}^{(n)}$, for $n$ large enough. This statistical metric is a relevant parameter regarding the \emph{hardness} of the classification. The larger the identifiability gap, the \emph{easier} the classification via thresholding of the entries of the matrix $\mathcal{F}^{(n)}$ tends to be.

Similarly, define the identifiability gap $\Gamma\left(T^{(n)}\right)$ associated with a tensor-valued estimator $T^{(n)}$ as the maximum distance among all parallel hyperplanes that consistently separate the features, as in Fig.~\ref{fig:scheme}. For example, the SVM algorithm is designed to find these margins. More concretely,
\begin{equation}
\Gamma\left(T^{(n)}\right) \overset{\Delta}= \max\limits_{\left(\mathbf{w},\tau_1\right),\left(\mathbf{w},\tau_2\right)\in \mathcal{H}} \frac{\left|\tau_1-\tau_2\right|}{\left|\left|\mathbf{w}\right|\right|},
\end{equation}
where~$\mathcal{H}$ indexes the set of linear maps that consistently separate the features: $\left(\mathbf{w},\tau\right)\in\mathcal{H}$ if and only if the linear map $\mathcal{L}_{\mathbf{w},\tau}(\mathbf{x}):=\mathbf{w}\cdot \mathbf{x}-\tau$ consistently separates the features.

\begin{lemma}\label{lem:stacking}
Let $T^{(n)}$ be a tensor-valued estimator whose underlying features at each pair $ij$ are defined as
\begin{equation}
\mathcal{T}^{(n)}_{ij}:=\left(\left[E^{(n)}_1\right]_{ij},\left[E^{(n)}_2\right]_{ij},\ldots,\left[E^{(n)}_{M}\right]_{ij}\right)\in \mathbb{R}^{M},
\end{equation}
with identifiability gap~$\Gamma_E^{(n)}\overset{\Delta}=\Gamma\left(T^{(n)}\right)$. Let~$\widehat{A}^{(n)}$ be a matrix-valued estimator with identifiability gap~$\Gamma_{A}^{(n)}\overset{\Delta}=\Gamma\left(\widehat{A}^{(n)}\right)$. If both $\widehat{A}^{(n)}$ and $T^{(n)}$ are (linearly) structurally consistent w.h.p., then the tensor-valued estimator $\widetilde{T}^{(n)}$ defined via the \emph{augmented} features
\begin{equation}
\widetilde{\mathcal{T}}^{(n)}_{ij}:=\left(\left[\widehat{A}^{(n)}\right]_{ij},\left[E^{(n)}_1\right]_{ij},\ldots,\left[E^{(n)}_{M}\right]_{ij}\right)\in \mathbb{R}^{M},
\end{equation}
exhibits an identifiability gap obeying $\Gamma\left(\widetilde{T}^{(n)}\right)\geq \left|\left|\Gamma^{(n)}\right|\right|_{2}$ w.h.p., with $\Gamma^{(n)}\overset{\Delta}=\left(\Gamma_{A}^{(n)},\Gamma_{E}^{(n)}\right)$.
\end{lemma}

Lemma~\ref{lem:stacking} asserts that, if further matrix-valued structurally consistent estimators are incorporated into the feature vector, the identifiability gap increases.

\begin{proof}
Let ${\sf Cv}\left(\mathcal{S}\right)$ denote the \emph{convex hull} of a set $\mathcal{S}\subset \mathbb{R}^M$, i.e., the smallest convex set containing~$\mathcal{S}$~\cite{convex}. Define $\widetilde{\mathcal{C}}\overset{\Delta}=\left\{\widetilde{\mathcal{T}}^{(n)}_{ij}\right\}_{ij\,:\,A_{ij}\neq 0}$ as the set of augmented features associated with connected pairs and $\widetilde{\mathcal{D}}\overset{\Delta}=\left\{\widetilde{\mathcal{T}}^{(n)}_{ij}\right\}_{ij\,:\,A_{ij}=0}$ associated with disconnected pairs. Similarly, define $\mathcal{C}\overset{\Delta}=\left\{\mathcal{T}^{(n)}_{ij}\right\}_{ij\,:\,A_{ij}\neq 0}$ and $\mathcal{D}\overset{\Delta}=\left\{\mathcal{T}^{(n)}_{ij}\right\}_{ij\,:\,A_{ij}=0}$. Let $R$ be the smallest entry of $\widehat{A}^{(n)}$ across connected pairs and $r$ be the greatest entry of $\widehat{A}^{(n)}$ across disconnected pairs and note that $r<R$ w.h.p, since $\widehat{A}^{(n)}$ is structurally consistent w.h.p. We have that
\begin{equation}
\begin{array}{ccl}
\Gamma\left(\widetilde{T}^{(n)}\right)^2 \!\!\!\!& \overset{(a)}= & \!\!\!\! {\sf d}\left({\sf Cv}\left(\widetilde{\mathcal{C}}\right),{\sf Cv}\left(\widetilde{\mathcal{D}}\right)\right)^2\\
\!\!\!\!& \overset{(b)}\geq & \!\!\!\! {\sf d}\left({\sf Cv}\left(\mathcal{C}\times \!\!\left.\left[R,\infty\right.\right)\right),{\sf Cv}\left(\mathcal{D}\times \!\! \left.\left(-\infty,r\right.\right]\right)\right)^2\\
\!\!\!\!& \overset{(c)}\geq & \!\!\!\! {\sf d}\left({\sf Cv}\left(\mathcal{C}\right)\times \!\! \left.\left[R,\infty\right.\right),{\sf Cv}\left(\mathcal{D}\right)\times \!\! \left.\left(-\infty,r\right.\right]\right)^2\\
\!\!\!\!& \overset{(d)}= & \!\!\!\! {\sf d}\left({\sf Cv}\left(\mathcal{C}\right),{\sf Cv}\left(\mathcal{D}\right)\right)^2+(R-r)^2\\
\!\!\!\!& = & \!\!\!\!\left(\Gamma^{(n)}_E\right)^2+\left(\Gamma^{(n)}_{A}\right)^2=\left|\left|\Gamma^{(n)}\right|\right|_2^2
\end{array}\nonumber
\end{equation}
where for two subsets $\mathcal{X},\mathcal{Y}\subset{R}^{M}$,  $d\left(\mathcal{X},\mathcal{Y}\right)$ is the distance
\begin{equation}
{\sf d}\left(\mathcal{X},\mathcal{Y}\right)\overset{\Delta}=\inf_{x\in \mathcal{X},y\in \mathcal{Y}}\left|\left|x-y\right|\right|_2;
\end{equation}
the first identity~$(a)$ conforms to an alternative dual characterization for the identifiability gap (refer to Theorem~$13$ in~\cite{DistanceConvex});~$(b)$ holds in view of the inclusions $\widetilde{\mathcal{C}}\subset \mathcal{C}\times \left.\left[R,\infty\right.\right)$ and $\widetilde{\mathcal{D}}\subset \mathcal{D}\times \left.\left(-\infty,r\right.\right]$; $(c)$ holds since $\mathcal{C}\subset {\sf Cv}\left(\mathcal{C}\right)$ and $\mathcal{D}\subset {\sf Cv}\left(\mathcal{D}\right)$ and the fact that the convex hull commutes with the Cartesian product over convex sets~\cite{convex}; the identity~$(d)$ is straightforward from the definition of the distance~${\sf d}(\cdot,\cdot)$.
\end{proof}

\section{Methodology}

In order to stratify the pairs of nodes into connected or disconnected from the observed time series, we address the linear separability property of the covariance-based features~$\left\{\mathcal{T}^{(n)}_{ij}\right\}_{ij}$ established in Theorem~$1$, by studying the performance of trained classifiers, in particular, linear Support Vector Machines (SVMs) and Convolutional Neural Networks (CNNs).
The training set is given by
\begin{equation}
{\sf Tr}^{(n)}\overset{\Delta}=\left\{\left(\overline{\mathcal{T}}^{(n)}_{ij},\mathbf{1}_{\left\{A_{ij}\neq 0\right\}}\right)\right\}_{i\neq j}
\end{equation}
\noindent where we have introduced the normalized feature vectors
\begin{equation}\label{eq:norma}
\overline{\mathcal{T}}^{(n)}_{ij}:=\frac{\mathcal{T}^{(n)}_{ij}}{\max_{i\neq j}\left|\left|\mathcal{T}^{(n)}_{ij}\right|\right|_{\infty}},
\end{equation}
\noindent with the unnormalized features given by,
\begin{equation}
\mathcal{T}_{ij}^{(n)}\overset{\Delta}=\left(\left[\widehat{R}_{-100}(n)\right]_{ij},\left[\widehat{R}_{-99}(n)\right]_{ij},\ldots,\left[\widehat{R}_{100}(n)\right]_{ij}\right).\nonumber
\end{equation}
In other words, for training, we provide a normalized feature~$\overline{\mathcal{T}}^{(n)}_{ij}$ associated with the pair $ij$ as input to a classifier and the output should be the ground truth $1_{\left\{A_{ij}\neq 0\right\}}$.

The normalization in the training set is motivated by the following observation. With infinitely many samples,
\begin{equation}
\mathcal{T}^{\infty}_{ij}=\sigma^2\left(\left[\overline{R}_D\right]_{ij},\left[\overline{R}_{D+1}\right]_{ij},\ldots,\left[\overline{R}_M\right]_{ij}\right)
\end{equation}
where~$\overline{R}_{k}$ is the $k$-lag covariance matrix (equation~\eqref{eq:truecov}) of the normalized process $\left(\mathbf{y}(n)/\sigma\right)_{n\in\mathbb{N}}$, i.e., the process whose noise is normalized to unit variance. With the proposed normalization in equation~\eqref{eq:norma}, the multiplicative factor $\sigma^2$ is cancelled out, which decreases the role played by the noise-level in the performance of the trained CNNs. Furthermore, this normalization renders the generalization performance of the trained CNNs robust across structurally distinct graphs.

To generate the matrix $A$ to obtain the time series data~$\left\{\mathbf{y}(\ell)\right\}_{\ell=0}^{n}$, given a graph $G$, the following procedure was considered. Let $G$ be a given graph without self-loops, i.e., $G_{ii}=0$ for all $i$. Define the interaction matrix $A$ as
\begin{equation}
\left\{\begin{array}{ccll}
A_{ij} & = & \alpha_1 \frac{G_{ij}}{d_{\max}(G)}, & \mbox{ for } i\neq j\\
A_{ii} & = & \alpha-\sum_{k\neq i} A_{ik}, & \mbox{ for all } i
\end{array},\right.
\end{equation}
where $d_{\max}(G)$ is the maximum \emph{in-flow} degree of the underlying graph $G$ and $0<\alpha_1\leq\alpha<1$ are some constants. In other words, the rows of $A$ sum to $\alpha<1$ and its support is given by $G$. This is often cast as the \emph{Laplacian rule}~\cite{Sayed}. This interaction matrix renders the networked dynamical system~\eqref{eq:model} stable and with a support graph of interactions given by $G$. To generate the support graph $G$, we considered the realization of random graph models as Erdős–Rényi for undirected graphs, binomial random graphs for directed graphs, and real-world networks.

We train the classifiers over one realization of a random graph model with~$p=0.5$ and~$N=100$ and apply them to distinct networks, including real-world ones, where $p$ is the probability of edge or arrow drawing in the random graph model and $N$ is the number of nodes. Throughout, we assume that we can only observe the time series data from $\left|S\right|=20$ nodes, that is, we assume $S=\left\{1,2,\ldots,20\right\}$.

\section{Simulation Results}

In the numerical results considered, we define \emph{accuracy} as the number of directed pairs correctly classified over the total number of directed pairs in the underlying graph. We consider $1000$ Monte Carlo runs across all plots.

\begin{figure}
\centering
\begin{subfigure}[b]{0.45\textwidth}
   \includegraphics[width=1\linewidth]{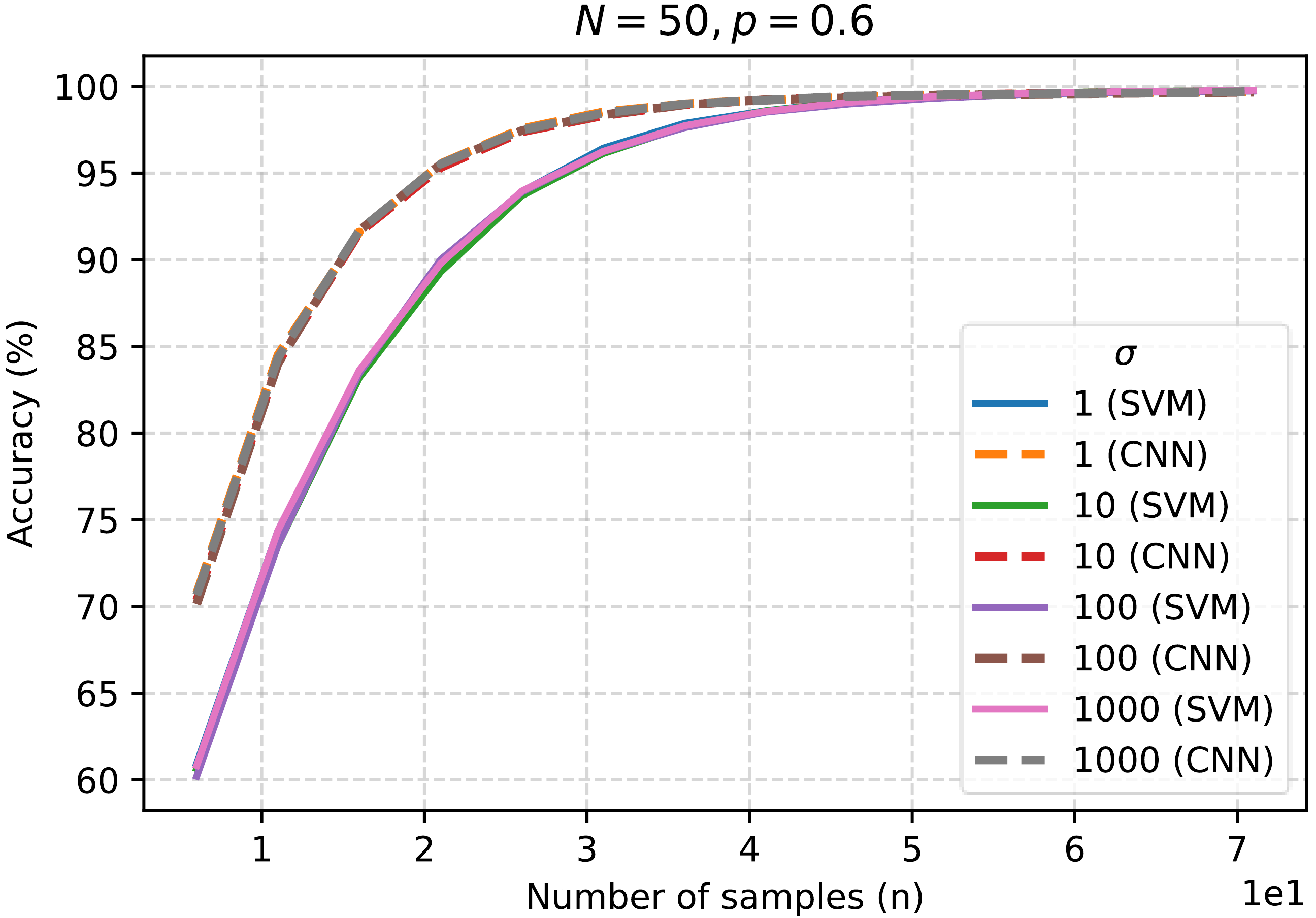}
   \caption{}
   \label{fig:Ng1}
\end{subfigure}

\begin{subfigure}[b]{0.45\textwidth}
   \includegraphics[width=1\linewidth]{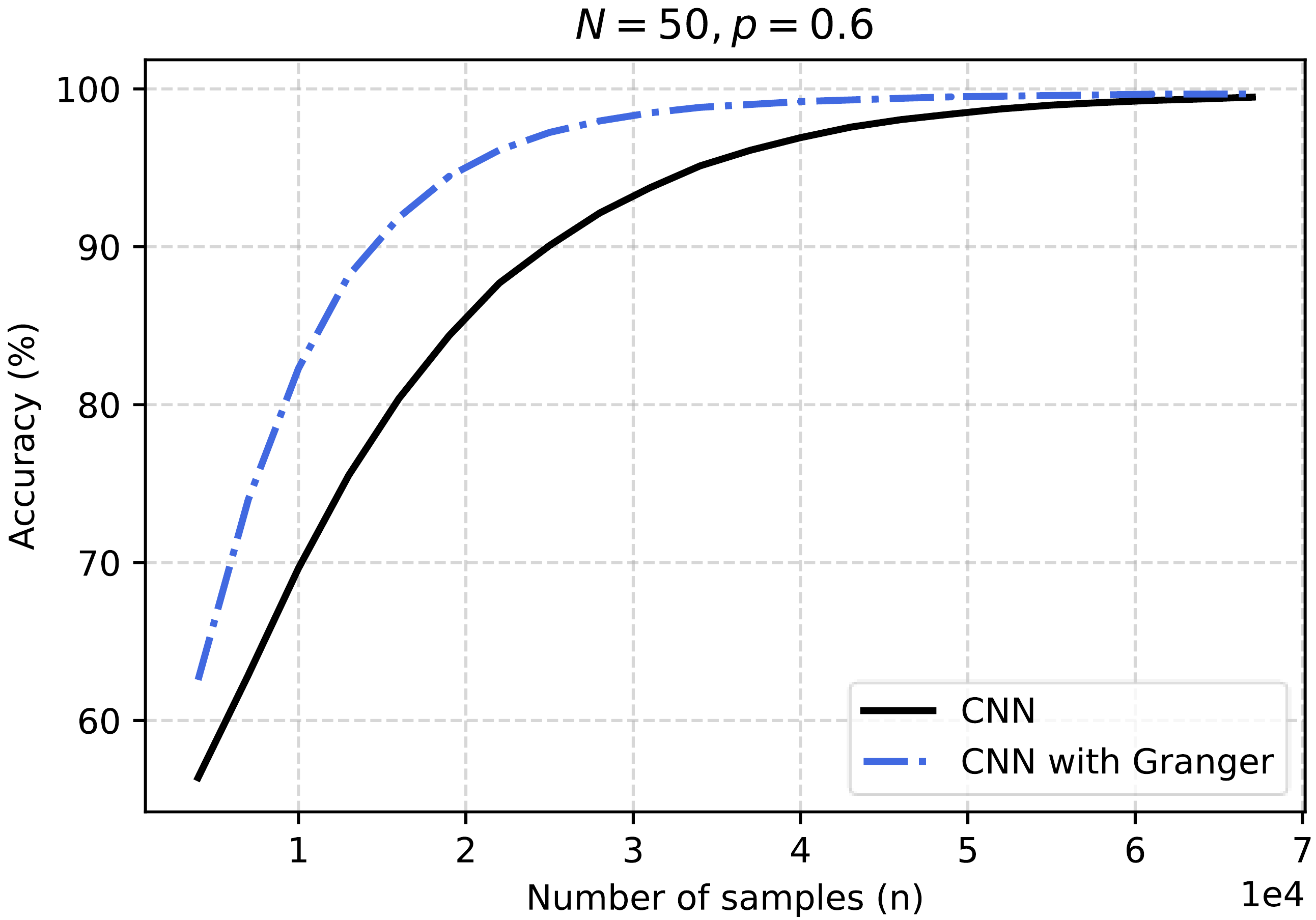}
   \caption{}
   \label{fig:Ng2}
\end{subfigure}

\caption{(a) depicts the robustness against the noise variance for both the CNN and the SVM; and (b) considers the inclusion of the Granger estimator in the feature vector. }
%\vspace{-10px}
\end{figure}

Fig.~\ref{fig:Performance} $(a)-(f)$ depict the sample-complexity performance of the estimators across structurally distinct networks, considering: $i$) Granger under partial observability~$\left[\widehat{R}_{1}(n)\right]_{\mathcal{S}}\left(\left[\widehat{R}_0(n)\right]_{\mathcal{S}}\right)^{-1}$ that is provably structurally consistent~\cite{tomo_journal,tomo_journal_proceedings} for distinct regimes of network connectivity; $ii$) The one-lag estimator~$\widehat{R}_1(n)$, which is also consistent for several network connectivity regimes~\cite{open_Journal}; $iii$) the $\widehat{R}_1(n)-\widehat{R}_3(n)$ that is structurally consistent~\cite{R1minusR3}; $iv$) the linear SVM; and $v$) the trained CNNs. For classification, we apply Gaussian mixture over the sorted entries of the matrix-valued estimators in order to stratify the connected and disconnected pairs. Our results show the overall superiority in performance for the CNN-based classifier. Figs.~\ref{fig:Performance} $(a)-(b)$ refer to real-world networks obtained from the database~\cite{nr}, with $(a)$ for a Brain structural connectivity matrix of a monkey and $(b)$ for an enzyme biochemical network; $(c)-(d)$ refer to symmetric regimes where the underlying support graph of the networked dynamical system is undirected; and $(e)-(f)$ refer to directed graph regimes. $N$ and $p$ stand for the number of nodes and probability of edge/arrow drawing in the random graph models. It should be remarked that while the CNN is trained over a synthetic network, namely, a particular realization of an Erdős–Rényi (for undirected networks) random graph model with~$p=0.5$ and~$N=100$, it generalizes well over real-world networks as demonstrated in Figs.~\ref{fig:Performance} $(a)-(b)$.

Fig.~\ref{fig:Ng1} illustrates the robustness of both the trained CNNs and the linear SVMs against distinct noise-level regimes. The CNNs and SVMs are trained with a noise variance of $\sigma^2=0.5$, but generalize well over an extended range of noise variance. Fig.~\ref{fig:Ng1} shows that the performance of these classifiers is not sensitive to the variance of the input noise in the dynamics~\eqref{eq:model}. Fig.~\ref{fig:Ng2} shows the gain in performance when the Granger estimator is included in the feature vector. In particular, when we include in the feature vector
\begin{equation} \mathcal{T}_{ij}^{(n)}\overset{\Delta}=\left(\left[\widehat{A}_{\mathcal{S}}\right]_{ij},\left[\widehat{R}_{-100}(n)\right]_{ij},\ldots,\left[\widehat{R}_{100}(n)\right]_{ij}\right)\nonumber
\end{equation}
the additional component $\widehat{A}_{\mathcal{S}}\overset{\Delta}=\left[\widehat{R}_1(n)\right]_{\mathcal{S}}\left(\left[\widehat{R}_0(n)\right]_{\mathcal{S}}\right)^{-1}$
that is the Granger under partial observability, with only $\left|S\right|=20$ nodes observed.
This is consistent with Lemma~\ref{lem:stacking}, motivating the search for feature vectors built on other matrix-valued structurally consistent estimators. It motivates the following causal inference paradigm: $i$) characterize matrix-valued structurally consistent estimators; $ii$) define feature vectors that collect these consistent estimators; $iii$) use these new features to train classifiers like a CNN.

\section{Concluding Remarks}

This paper considered the problem of determining the graph that captures the fundamental dependencies among time series of data. These time series are indexed as nodes in linear stochastic networked dynamical systems. Only the time series of some nodes are observed (partial observability). We proposed a novel feature-based paradigm and proved that the features were consistently linearly separable. With this separability property, our features can be used as an input to a variety of machine learning pipelines in order to design new state-of-the-art algorithms for causal inference of linear networked dynamical systems. In particular, CNNs trained over this set of features exhibited remarkable sample-complexity performance, significantly reducing the number of samples required to reach a certain level of accuracy, as compared with other state-of-the-art estimators, which require a much larger number of samples. Simulation results show the superiority of the CNN-based approach. It was further shown that the inclusion of structurally consistent matrix-valued estimators in the feature vectors increases the performance of structure identification. This motivates further study of new structurally consistent matrix-valued estimators as building blocks for feature vectors or tensor-valued estimators.

%We proposed a set of covariance-based features and proved they are consistently linearly separable.

%\clearpage

\section{Acknowledgments}
The work of S. Machado, A. Santos, J. Henriques and P. Gil was funded in part by the FCT - Foundation for Science and Technology, Portugal, I.P./MCTES through national funds (PIDDAC), within the scope of CISUC
R\&D Unit - UIDB/00326/2020 or project code UIDP/00326/2020 and CTS - Centro de Tecnologia e Sistemas - UIDB/00066/2020. The work of José M. F. Moura was funded in part by the U.S. National Science Foundation under Grant CCN 1513936. The work of Anirudh Sridhar was funded by the U.S. National Science Foundation under Grants CCF-1908308 and ECCS-2039716, and a grant from the C3.ai Digital Transformation Institute. The source code for the numerical experiments can be found at https://github.com/ASanctvs/Structure-Identification.

{\small
	\bibliographystyle{aaai23}
	\bibliography{biblio}
}

\end{document}